\documentclass[letterpaper, 10 pt, conference]{ieeeconf}

\IEEEoverridecommandlockouts                              %

\usepackage{ifthen}
\newboolean{showcomments}
\setboolean{showcomments}{false}
\usepackage[usenames,dvipsnames]{color}

\definecolor{bleudefrance}{rgb}{0.19, 0.55, 0.91}
\definecolor{ao(english)}{rgb}{0.0, 0.5, 0.0}

\newcommand{\addcite}[0]{\ifthenelse{\boolean{showcomments}}
{\textcolor{purple}{(add cite(s)) }}{}}%

\newcommand{\addcites}[1]{\ifthenelse{\boolean{showcomments}}
{\textcolor{blue}{(add cite(s)) #1}}{}}%

\newcommand{\you}[1]{  \ifthenelse{\boolean{showcomments}}
{{\color{blue} PCY: #1}}{}}
\newcommand{\youmargin}[1]{\ifthenelse{\boolean{showcomments}}{\marginpar{\color{bleudefrance}\tiny PCY: #1}}{}}

\newboolean{showedits}
\setboolean{showedits}{true}
\usepackage[markup=underlined]{changes}
\definechangesauthor[color=bleudefrance]{PCY}
\newcommand{\ayou}[1]{
\ifthenelse{\boolean{showedits}}
{\added[id=PCY]{#1}}
{\!#1\hspace{-4.75pt}}
}
\newcommand{\ryou}[2]{
\ifthenelse{\boolean{showedits}}
{\replaced[id=PCY]{#1}{#2}}
{\!#1\hspace{-4.75pt}}
}
\newcommand{\dyou}[1]{
\ifthenelse{\boolean{showedits}}
{\deleted[id=PCY]{#1}}
{}
}
\newcommand{\hl}[1]{\ifthenelse{\boolean{showcomments}}
{\textcolor{purple}{#1}}{}}%
\newcommand{\AC}[1]{\ifthenelse{\boolean{showcomments}}{\marginpar{\color{bleudefrance}\tiny AC: #1}}{}}

\usepackage[labelformat=simple]{subcaption}

\graphicspath{{figures/}}

\usepackage{epsfig, latexsym, times, bbm}
\usepackage{color}
\usepackage{colortbl}
\usepackage{array}
\usepackage{cite}

\usepackage{amsmath}
\usepackage{graphics}
\usepackage{graphicx}
\usepackage{subcaption} %
\usepackage{amssymb}
\usepackage{url}
\usepackage{comment}

\usepackage{amsthm}

\usepackage{multirow}

\usepackage{bbding}
\usepackage{amsfonts}
\usepackage{cases}

\usepackage{algorithmic}
\usepackage{mathrsfs}
\usepackage{bm}
\usepackage{verbatim}
\usepackage{enumerate}

\usepackage{booktabs}
\usepackage[linesnumbered,ruled,longend]{algorithm2e}
\usepackage{textcomp}
\usepackage{subeqnarray}
\usepackage{dsfont,optidef}

\usepackage{makecell}

\newcommand{\PreserveBackslash}[1]{\let\temp=\\#1\let\\=\temp}
\newcolumntype{C}[1]{>{\PreserveBackslash\centering}p{#1}}
\newcolumntype{R}[1]{>{\PreserveBackslash\raggedleft}p{#1}}
\newcolumntype{L}[1]{>{\PreserveBackslash\raggedright}p{#1}}

\renewcommand\thepage{}

\allowdisplaybreaks

\usepackage{cleveref}
\crefname{figure}{Fig.}{Figs.}

\let\labelindent\relax
\usepackage{enumitem}

\usepackage{pifont}
\def\ba{\begin{array}}
	\def\ea{\end{array}}
\newcommand{\beq}{\begin{equation}}
	\newcommand{\eeq}{\end{equation}}
\newcommand{\bq}{\begin{eqnarray}}
	\newcommand{\eq}{\end{eqnarray}}
\newcommand{\bqn}{\begin{eqnarray*}}
	\newcommand{\eqn}{\end{eqnarray*}}
\newcommand{\bee}{\begin{enumerate}}
	\newcommand{\eee}{\end{enumerate}}
\newcommand{\bi}{\begin{itemize}}
	\newcommand{\ei}{\end{itemize}}
\newcommand{\bseq}{\begin{subequations}}
\newcommand{\eseq}{\end{subequations}}

\newcommand{\bbR}{\mathbb{R}}

\newcommand{\bbQ}{\mathbb{Q}}

\newcommand{\bbP}{\mathbb{P}}

\newcommand{\calD}{\mathcal{D}}

\newcommand{\calM}{\mathcal{M}}
\newcommand{\calN}{\mathcal{N}}
\newcommand{\calO}{\mathcal{O}}

\newcommand{\calW}{\mathcal{W}}
\newcommand{\calX}{\mathcal{X}}
\newcommand{\calY}{\mathcal{Y}}
\newcommand{\calZ}{\mathcal{Z}}

\newcommand{\ev}{\mathbb{E}}

\newcommand{\toodo}[1]{  \noindent\ifthenelse{\boolean{showcomments}}
{{\color{red} To-do: #1}}{}}

\newcommand{\oprocendsymbol}{\hbox{$\bullet$}}
\newcommand{\oprocend}{\relax\ifmmode\else\unskip\hfill\fi\oprocendsymbol}

\newtheorem{theorem}{{Theorem}}%
\newtheorem{lemma}[theorem]{{Lemma}}
\newtheorem{definition}{{Definition}}

\newtheorem{proposition}[theorem]{{Proposition}}

\newtheorem{assumption}{{Assumption}}

\newtheorem*{inferring procedure}{{Inferring Procedure}}

\renewcommand{\addcites}[1]{\ifthenelse{\boolean{showcomments}}{(add cite(s)) #1}{}}

\usepackage{enumitem}
\renewcommand{\baselinestretch}{.99}

\begin{document}

\title{Distributionally Robust Federated Learning with Multi-Source Data}

\author{Yingzhu Liu, Zhongkui Li, Pengcheng You\textcolor{black}{$^{\dagger}$}, Ashish Cherukuri %
\thanks{Y. Liu, Z. Li, and P. You are with the Department of Control Science and Systems Engineering, Peking University, Beijing, China. A. Cherukuri is with the Engineering and Technology Institute Groningen and the Jan C. Willems Center for Systems and Control, University of Groningen, The Netherlands. \textcolor{black}{This work was supported in part by the National Natural Science Foundation of China (NSFC) under grants 72671003, 72431001, and 62373008. AC was supported in part by TKI HTSM FD-CODE (24PPS175).}}
\thanks{\textcolor{black}{$^{\dagger}$Corresponding author: Pengcheng You.}}
}

\maketitle

\pagestyle{empty}  %
\thispagestyle{empty} %
\begin{abstract}
Federated learning trains a shared model from private client data. In practice, data-generating distributions may differ, and the true mixture across clients is often unknown, making the underlying group distribution difficult to specify.
Existing approaches address cross-client mixture uncertainty by optimizing against the worst-case mixture, yet assume accurate client-wise distribution estimates.
However, these estimates can be unreliable when based on finite samples.
To handle both cross-client mixture uncertainty and within-client distributional ambiguity, we construct a global ambiguity set as the union of admissible mixtures of local ambiguity sets.
The construction allows client-specific ambiguity radii and admits a client-wise separable reformulation.
Leveraging this structure, we establish a high-probability out-of-sample performance guarantee.
We further develop a federated algorithm for a penalty-based reformulation and prove its convergence under milder regularity conditions.
Simulations validate the algorithm’s effectiveness.
\end{abstract}

\IEEEpeerreviewmaketitle

\section{Introduction}
In a federated learning (FL) system, a central server coordinates multiple clients to train a shared model while keeping data private.
This decentralized and privacy-preserving paradigm has been widely applied in mobile devices and healthcare networks~\cite{antunes2022federated}.
In practice, varying measurement conditions across clients often lead to heterogeneous local data-generating distributions. As a result, two sources of uncertainty arise. First, the true global distribution, under which the learned model is expected to perform well, is unknown. This is assumed to be a mixture of the clients' local distributions, but the mixing ratio is unknown. Second, since local samples are limited, the local datasets might not accurately represent the true local distributions. These two challenges hinder the development of robust models that generalize well.

To address the within-client distributional ambiguity, we adopt the framework of Distributionally Robust Optimization (DRO)~\cite{mohajerinesfahani2018datadriven,kuhn2019wasserstein,cherukuri2020cooperative}. DRO optimizes against the worst-case expected cost over an ambiguity set, that is, a family of candidate distributions consistent with the observed data. 
Among various choices of ambiguity sets, the Wasserstein ball has gained significant attention due to its ability to capture geometric feature shifts while often admitting tractable finite-dimensional reformulations~\cite{mohajerinesfahani2018datadriven}. 

In multi-source settings, a key challenge is how to define a coherent global learning objective. Classical federated algorithms, such as FedAvg~\cite{mcmahan2017communication}, assign weights in proportion to local sample sizes. 
However, this empirical approach often fails to match the true mixture weights in practice, potentially leading to biased models.
To address this issue, several approaches instead focus on the worst-case performance.
For instance, agnostic federated learning (AFL) considers the worst-case mixture of local objectives to improve robustness and fairness~\cite{mohri2019agnostica}. From a distribution-level perspective, Group DRO similarly optimizes over the worst-case mixture of local distributions under group-level shifts~\cite{soma2022near,yu2024efficient}. 
However, both lines of work presume accurate local distribution estimates, which can be fragile in the finite-sample regime.
While recent work incorporates local ambiguity into Group DRO models~\cite{konti2025group}, its solution is limited to centralized settings, and extending it to federated learning remains non-trivial.

Another line of work, not necessarily restricted to federated learning,  models uncertainty in multi-source data through a global ambiguity set centered at a representative distribution, such as a Wasserstein barycenter~\cite{rychener2024wasserstein} or a specific mixture of distributions~\cite{nguyen2022generalization}.
One recent variant uses an unbalanced Wasserstein ambiguity set for outlier robustness, with the center varying over admissible mixtures of local distributions~\cite{wang2025distributionally}. However, both the single-center formulation and this unbalanced Wasserstein variant couple local uncertainties together. As a result, the relation between the ambiguity set and the local distributions is less transparent, which hinders explicit out-of-sample performance analysis.
By comparison, a mixture of Wasserstein balls preserves local structure more directly as in~\cite{ibrahim2025fdrsvm}, but the formulation in that work is restricted to fixed mixture weights and therefore does not account for cross-client mixture uncertainty.

To address these challenges, we develop a distributionally robust FL framework based on admissible mixtures of client-wise Wasserstein ambiguity sets. We establish out-of-sample guarantees for the resulting global ambiguity set and develop BiDRO-FL for a penalized formulation.
Our main contributions are summarized as follows:
\begin{enumerate}
    \item We design a global ambiguity set capturing the two kinds of uncertainty: the cross-client mixture uncertainty and the client-specific ambiguity in estimating the true local distribution using local samples.
    \item We derive a lower bound on the probability that the global ambiguity set contains the true global distribution and an out-of-sample performance bound for the resulting robust solution under the true distribution.
    \item We develop a federated algorithm with provable convergence for a Lagrangian penalty reformulation.
\end{enumerate}
\ifdefined\shortpaper\fi

\section{Problem Statement}
Consider a federated learning system with $n$ clients indexed by $[n]:=\{1,\ldots,n\}$.
Each client  computes model updates locally using private data and communicates only with a central server to train a shared model.
Let $\bbP$ denote the \emph{unknown} true group distribution of clients. 
Let 
$\xi$ 
be a measurable random variable in the space $\Xi\subseteq\mathbb R^m$ that follows the distribution $\bbP$, where   the set $\Xi$ is convex and compact.
Note that the true distribution $\bbP$ is assumed to be a mixture of local distributions:
\begin{align}\label{eq:mixture4ground_truth}
    \bbP=\sum_{i=1}^n \alpha_i^\star \bbP_i\,,
\end{align}
where $\alpha^\star:=(\alpha^\star_1,\ldots,\alpha_n^\star)\in\Delta_{n-1}$ is \emph{unknown}, and $\bbP_i$ is the true local distribution for client $i\in[n]$. Here, $\Delta_{n-1}:=\{\alpha\in\bbR^n:\alpha_i\geq0,\ \sum_{i=1}^n\alpha_i=1\}$ is the probability simplex.
For each client $i$, the samples in the local dataset $\calD_i:=\{\widehat{\xi}_i^1,\ldots,\widehat{\xi}_i^{N_i}\}$ are drawn independently from the \emph{unknown} local distribution $\bbP_i$. The datasets $\calD_1,\ldots,\calD_n$ are mutually independent.
Note that $\alpha^\star_i$ may differ from the empirical proportions
 $N_i/N$, where $N:=\sum_iN_i$ is the total number of samples.
The goal is to learn a shared model $x\in\calX$ from data distributed across clients, where the learning part is characterized by a cost function $\ell:\bbR^d\times\Xi\rightarrow\bbR, (x,\xi)\mapsto\ell(x,\xi)$, which is continuous on $\calX\times\Xi$.
The performance of the learned model is evaluated under the true group distribution $\bbP$, i.e., $\ev_{\bbP}[\ell(x,\xi)]$.
Here, the set $\calX$ is convex and compact with $D_x:=\max_{x,x^\prime\in\calX}\|x-x^\prime\|$, where $\|\cdot\|$ denotes the Euclidean norm in the appropriate dimension throughout the paper.

\subsection{Within-client distributional ambiguity}
Given the observed samples, each client can approximate $\bbP_i$ by the empirical distribution $\widehat{\bbP}_i$ from local data, given by
\begin{align}\label{eq:empirical_distri}
    \widehat{\bbP}_i:=\frac{1}{N_i}\sum_{k=1}^{N_i}\delta_{\widehat{\xi}_i^k}\,,
\end{align}
where $\delta_{\widehat{\xi}_i^k}$ is the unit point mass at $\widehat{\xi}_i^k$.
However, such an estimator may not accurately represent $\bbP_i$ when the number of gathered samples is small. 
Additionally, the single-point estimate may fail to capture plausible local perturbations around the true distribution $\bbP_i$.
To tackle both issues, client $i$ constructs a set of  plausible distributions, referred to as an ambiguity set, and then makes decisions against the worst-case distribution in the set.
Specifically, for a fixed $p\in[1,\infty)$, for each client $i\in[n]$, we define the local ambiguity set as a $p$-Wasserstein ball of radius $\varepsilon_i$ centered at $\widehat{\bbP}_i$:
\begin{align}\label{eq:local_ambiguity_set}
    \calW_i(\varepsilon_i):=\Bigl\{\bbQ\in\calM(\Xi):W_p(\bbQ,\widehat{\bbP}_i)\leq\varepsilon_i\Bigr\}\,,
\end{align}
where $\calM(\Xi)$ is the  space of probability distributions $\bbQ$ supported on $\Xi$ with finite $p$-th moment.
Moreover, $\varepsilon_i\ge 0$ is a client-dependent ambiguity radius, allowing different levels of local uncertainty across clients. 
To measure the distance between a candidate distribution \(\bbQ\) and the empirical distribution \(\widehat{\bbP}_i\), we use the \(p\)-Wasserstein distance, denoted by \(W_p(\cdot,\cdot)\) and defined below.
\begin{definition}[$p$-Wasserstein Distance~\cite{kuhn2019wasserstein}]~\label{def:Wp-metric}
For any $p\in[1,+\infty)$, let $\bbP$ and $\bbP^\prime$ be two probability distributions on a Polish metric space $(\Xi,d)$ with finite moments of order $p$.
Let $\Gamma(\bbP,\bbP^\prime)$ denote the set of all couplings $\gamma\in\calM(\Xi\times\Xi)$ having first marginal $\bbP$ and second marginal $\bbP^\prime$, and let~$c:\Xi\times\Xi\to[0,\infty),c(x,y):=d(x,y)^p$ denote the transportation cost.
The $p$-Wasserstein distance between $\bbP$ and $\bbP^\prime$ is defined as
\begin{align}\label{eq:def_wp}
W_p(\bbP,\bbP^\prime)
\;:=\;
\Bigl(\inf_{\gamma\in\Gamma(\bbP,\bbP^\prime)}
\int_{\Xi\times\Xi} c(x,y)\, \mathrm d\gamma(x,y)\Bigr)^{\frac{1}{p}}\,.
\end{align}
\end{definition}%
Throughout this paper, we use the Euclidean ground metric $d(\xi,\zeta)=\|\xi-\zeta\|$, so that $c(\xi,\zeta)=\|\xi-\zeta\|^p$.

\subsection{Cross-client mixture uncertainty}
Following the mixture structure in~\eqref{eq:mixture4ground_truth}, we allow each local distribution to vary within its Wasserstein ambiguity set. For fixed weights $\alpha\in\Delta_{n-1}$, this yields the mixture of local ambiguity sets:
\begin{align}\label{eq:global_ambiguity}
    \calW_{\mathrm M}(\alpha):=\Big\{\mathbb{Q}=\sum_{i=1}^n\alpha_i\mathbb{Q}_i:\ \mathbb{Q}_i\in\calW_i(\varepsilon_i),\forall i\in[n]\Big\}\,,
\end{align}
where $\alpha_i$ is client $i$'s weight.
Ideally, if the true weights $\alpha^\star$  were known, one would use $\mathcal{W}_{\mathrm M}(\alpha^\star)$ to construct the global ambiguity set.
However, since $\alpha^\star$ is unavailable in practice, a natural robust approach is to consider all mixtures induced by weights in the feasible set $\Delta_{n-1}$. Accordingly, we define the global ambiguity set as a union:
   \begin{align}
        \calW_{\mathrm G}
:&=\bigcup\nolimits_{\alpha\in\Delta_{n-1}}\calW_{\mathrm M}(\alpha)
\notag\\
&=\Big\{\sum_{i=1}^n\alpha_i\bbQ_i:\ \alpha\in\Delta_{n-1},\ \bbQ_i\in\calW_i(\varepsilon_i)\Big\}.\label{eq:global_ambiguity_set}
   \end{align}
This construction accounts for uncertainty in local distribution estimates through $\calW_i(\varepsilon_i)$ and in client proportions through the union over $\alpha\in\Delta_{n-1}$. See Fig.~\ref{fig:ambiguity_set} for the proposed global ambiguity set with $n=6$ clients.

\subsection{Problem formulation}
\begin{figure}[t]
    \centering
    \includegraphics[width=0.75\columnwidth]{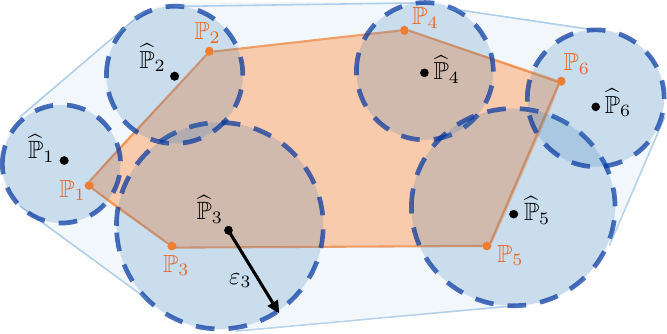}
    \caption{Illustration of the ambiguity set $\calW_{\mathrm G}$ defined in~\eqref{eq:global_ambiguity_set} for $n=6$ clients. For client $i$, the orange and black points denote $\bbP_i$ and $\widehat{\bbP}_i$, respectively. Blue dashed balls represent $\calW_i(\varepsilon_i)$, centered at $\widehat{\bbP}_i$. The orange region represents mixtures of the true client distributions.
    }
    \label{fig:ambiguity_set}
\end{figure}
With the global ambiguity set $\calW_{\mathrm G}$ in place, we consider the following mixing Wasserstein-ball  distributionally robust optimization (MW-DRO) problem:
\begin{align}\tag{MW-DRO}\label{eq:problem_DRO}
    \inf_{x\in\calX}\sup\limits_{\bbQ\in\calW_{\mathrm G}} \ev_{\bbQ}[\ell(x,\xi)]\,,
\end{align}
where $x\in\calX\subseteq\bbR^d$ is the decision variable. 
The goal is to seek a solution that is robust to both within-client distributional ambiguity and uncertainty in the client mixture weights.
Let $\widehat{x}^\star$ denote a DRO optimizer of~\eqref{eq:problem_DRO}, and define the DRO optimal value as $\widehat{J}^\star$.
We further define $J:=\ev_{\bbP}[\ell(\widehat{x}^\star,\xi)]$ as the true expected cost at the same decision $\widehat{x}^\star$ under the true distribution $\bbP$. 
This leads to two central questions in this paper:
\begin{enumerate}
    \item \textbf{Out-of-sample guarantee:} 
    How well does the DRO solution perform under the true distribution $\bbP$?
    \item \textbf{Efficient computation:} Can we solve the problem~\eqref{eq:problem_DRO} efficiently in a federated setting?
\end{enumerate}
\section{Out-of-Sample Performance Guarantees}
In this section, we evaluate the performance of the solution of~\eqref{eq:problem_DRO} under the true group distribution~$\bbP$. 
First, we establish a finite-sample \emph{coverage} guarantee ensuring that~$\bbP$ belongs to the proposed global ambiguity set with high probability. This coverage result yields an explicit bound on the \emph{out-of-sample performance gap}, i.e., the difference between the true expected cost achieved by a DRO optimizer and the DRO optimal value.
\subsection{Coverage of the ambiguity set}
The following \textcolor{black}{lemma} provides a local coverage guarantee by quantifying the probability that the true local distribution $\bbP_i$ belongs to the corresponding ambiguity set.
This result serves as a key building block for establishing coverage of the global ambiguity set.
\begin{lemma}[Measure Concentration~{\cite[Proposition 4.2]{boskos2024highconfidence}}]\label{lm:concentration}
    For each client $i$, consider the dataset $\calD_i$ sampled on a compact space $\Xi\subseteq\bbR^m$.
    \textcolor{black}{Then,} for any $p\geq 1$, $N_i\geq 1$, and confidence level $1-\beta_i$ with $\beta_i\in(0,1)$, we have:
        \begin{align}\label{eq:locao_coverage}
            \bbP^{N_i}\{\bbP_i\in \calW_i\!(\varepsilon_i^{N_i}(\beta_i))\!\}&=\bbP^{N_i}\{W_p(\bbP_i,\widehat{\bbP}_i)\leq \varepsilon_i^{N_i}(\beta_i)\}\nonumber\\
            &\geq 1-\beta_i\,,
        \end{align}
    where
\begin{small}
    \begin{align}\label{eq:radius}
    \varepsilon_i^{N_i}(\beta_i):=\begin{cases}
R\left[h^{-1}\!\left(\dfrac{\ln(c_1\beta_i^{-1})}{c_2 N_i}\right)\right]^{\frac{1}{p}},
& \text{if } p=\dfrac{m}{2},\\[1.2ex]
R\left(\dfrac{\ln(c_1\beta_i^{-1})}{c_2N_i}\right)^{\frac{1}{\max\{2p,m\}}},
& \text{if } p\neq\dfrac{m}{2},
\end{cases}
\end{align}
\end{small}%
with $R := \frac{1}{2}\,\mathrm{diam}_{\infty}\!\big(\Xi\big)$ and $h(x) := x^2/\left(\ln(2+1/x)\right)^2,\; x>0$. The constants $c_1$ and $c_2$ depend only on $p$ and $m$. Here, $\mathrm{diam}_{\infty}$ denotes the diameter induced by the infinity norm.
\end{lemma}
Throughout the remainder of this section, we set $\varepsilon_i:=\varepsilon_i^{N_i}(\beta_i)$, as defined in~\eqref{eq:radius}.
To establish coverage of the global ambiguity set, a baseline guarantee requires simultaneous local coverage events, which can be conservative for nearly homogeneous local distributions. To exploit the distributional structure, we introduce the following bounded-heterogeneity assumption.
\begin{assumption}~\label{ass:bounded_heterogeneity}
    Suppose there exists a constant $\delta\ge 0$ such that for all $i,j\in[n]$, $
W_p(\mathbb P_i,\mathbb P_j)\le \delta$.
\end{assumption}
With small heterogeneity $\delta$, one local ambiguity set may cover all true distributions, yielding a tighter bound.
Proposition~\ref{prop:coverage} combines the structure-aware result with the baseline by taking their maximum.
\begin{proposition}[Coverage of the Global Ambiguity Set]~\label{prop:coverage}
Let Assumption~\ref{ass:bounded_heterogeneity} hold.
For any $p\geq 1$, given $N_i\geq 1$, $\beta_i\in(0,1)$, and the radius $\varepsilon_i$ in~\eqref{eq:radius} for each client $i\in[n]$,
\textcolor{black}{the global ambiguity set} $\calW_{\mathrm G}$ defined in~\eqref{eq:global_ambiguity_set} satisfies
\begin{align}\label{eq:global_coverage}
    \mathrm{Pr}\{\bbP\in\calW_{\mathrm G}\}\geq 1-\beta\,,
\end{align} 
where $(1-\beta):=\max\Big\{\prod_{i=1}^n(1-\beta_i),~1-\prod_{i=1}^n\bar{\beta}_i\Big\}$.
\textcolor{black}{If $\varepsilon_i>\delta$, define $\bar\beta_i\in(0,1)$ by}
\[
\varepsilon_i^{N_i}(\bar\beta_i)=\varepsilon_i-\delta.
\]
\textcolor{black}{If this equation has no solution in $(0,1)$ or $\varepsilon_i\le\delta$, set $\bar\beta_i=1$.}
\end{proposition}
\ifdefined\shortpaper\else
\begingroup

\begin{proof}
We establish global coverage through two sufficient conditions: simultaneous coverage of the corresponding local distributions, and coverage of all true local distributions by a single local ambiguity set. All probabilities below refer to the joint sampling of the client datasets.

\noindent\textit{First condition: simultaneous local coverage.}
Suppose that each local ambiguity set contains its corresponding true distribution. For each $i\in[n]$, define
\[
E_i:=\{\bbP_i\in\calW_i(\varepsilon_i)\}
=\{W_p(\bbP_i,\widehat{\bbP}_i)\le\varepsilon_i\}.
\]
On $\bigcap_iE_i$, the true mixture weights give
\[
\bbP=\sum_{i=1}^n\alpha_i^\star\bbP_i
\in\calW_{\mathrm M}(\alpha^\star)
\subseteq\calW_{\mathrm G}.
\]
Lemma~\ref{lm:concentration} yields $\Pr(E_i)\ge1-\beta_i$. Each $E_i$ depends only on $\calD_i$, so the independence of the datasets implies
\begin{align*}
\Pr\{\bbP\in\calW_{\mathrm G}\}
&\ge\Pr\Bigl(\bigcap_{i=1}^nE_i\Bigr)\\
&=\prod_{i=1}^n\Pr(E_i)
\ge\prod_{i=1}^n(1-\beta_i).
\end{align*}

\noindent\textit{Second condition: one local set covers all true distributions.}
Sufficiently small heterogeneity allows a single local ambiguity set to contain all true local distributions, provided that its empirical estimate is sufficiently accurate. Define
\[
E_i^\sharp:=\begin{cases}
\{W_p(\widehat{\bbP}_i,\bbP_i)\le\varepsilon_i-\delta\},
&\varepsilon_i>\delta,\\
\varnothing,&\varepsilon_i\le\delta.
\end{cases}
\]
On $E_i^\sharp$, Assumption~\ref{ass:bounded_heterogeneity} and the triangle inequality give, for every $j\in[n]$,
\begin{align}
W_p(\widehat{\bbP}_i,\bbP_j)
&\le W_p(\widehat{\bbP}_i,\bbP_i)+W_p(\bbP_i,\bbP_j)\notag\\
&\le(\varepsilon_i-\delta)+\delta=\varepsilon_i.
\end{align}
Thus, $\calW_i(\varepsilon_i)$ contains every $\bbP_j$. Convexity of $W_p^p$ in its distribution argument further yields
\[
W_p^p(\widehat{\bbP}_i,\bbP)
\le\sum_{j=1}^n\alpha_j^\star W_p^p(\widehat{\bbP}_i,\bbP_j)
\le\varepsilon_i^p,
\]
so the same local ambiguity set contains their true mixture $\bbP$. Since the simplex permits assigning all weight to client $i$, we have $\calW_i(\varepsilon_i)\subseteq\calW_{\mathrm G}$. Consequently, it suffices that $E_i^\sharp$ holds for at least one client.

If $\bar\beta_i<1$, its defining equation and Lemma~\ref{lm:concentration} give $\Pr(E_i^\sharp)\ge1-\bar\beta_i$. If $\bar\beta_i=1$, this inequality holds trivially, including when $\varepsilon_i\le\delta$ or the defining equation has no solution. Each $E_i^\sharp$ depends only on $\calD_i$, so independence gives
\begin{align*}
\Pr\{\bbP\in\calW_{\mathrm G}\}
&\ge\Pr\Bigl(\bigcup_{i=1}^nE_i^\sharp\Bigr)\\
&=1-\prod_{i=1}^n\Pr\bigl((E_i^\sharp)^c\bigr)\\
&\ge1-\prod_{i=1}^n\bar\beta_i.
\end{align*}
Both arguments bound the probability of the same global-coverage event from below. Taking the larger of the two bounds proves~\eqref{eq:global_coverage}.
\end{proof}
\endgroup
\fi
\subsection{Performance evaluation under the true group distribution}
We now connect the coverage result to the out-of-sample performance of the solution of~\eqref{eq:problem_DRO}.
The coverage implies that on the event $\{\bbP\in\calW_{\mathrm G}\}$, the worst-case expected cost in~\eqref{eq:problem_DRO} upper-bounds the true expected cost for any ${x\in\calX}$:
\begin{align}\label{eq:upper_worst}
    \ev_{\bbP}[\ell(x,\xi)] \;\le\; \sup_{\bbQ\in\calW_{\mathrm G}}\ev_{\bbQ}[\ell(x,\xi)]\,.
\end{align}
For the DRO solution $\widehat{x}^\star$, Proposition~\ref{prop:coverage} implies that the DRO optimal value $\widehat{J}^\star$ is a certificate for the true cost $J$ with high probability.
We next bound the out-of-sample performance gap $\widehat{J}^\star-J$ using the following auxiliary lemmas.
\begin{lemma}[\textcolor{black}{Simplex-weight inequality}]\label{lm:inequality}
    Let $\alpha^{\mathrm{I}},\alpha^{\mathrm{II}}\in\Delta_{n-1}$ and $b\in\bbR^n$, and let $\|\cdot\|_1$ denote the $\ell_1$-norm. Then, the following inequality holds:
    \begin{align}
        |(\alpha^{\mathrm{I}}-\alpha^{\mathrm{II}})^\top b|\leq\frac{1}{2}\|\alpha^{\mathrm{I}}-\alpha^{\mathrm{II}}\|_1\Bigl(\max_{i}b_i-\min_i b_i\Bigr)\,.
    \end{align}
\end{lemma}
\ifdefined\omitinequalitylemmaproof\else
\begin{proof}
    Since $\alpha^{\mathrm{I}},\alpha^{\mathrm{II}}\in\Delta_{n-1}$, we have $\sum_{i=1}^n(\alpha^{\mathrm{I}}_i-\alpha^{\mathrm{II}}_i)=0$.  Thus, for any $c\in\bbR$, we have
\begin{align}
            (\alpha^{\mathrm{I}}-\alpha^{\mathrm{II}})^\top b&=\sum_{i=1}^n(\alpha^{\mathrm{I}}_i-\alpha^{\mathrm{II}}_i)b_i-c\sum_{i=1}^n(\alpha^{\mathrm{I}}_i-\alpha^{\mathrm{II}}_i)\notag\\&=\sum_{i=1}^n(\alpha^{\mathrm{I}}_i-\alpha^{\mathrm{II}}_i)(b_i-c).\label{eq:lm_inequality}
\end{align}
    By choosing $c=(\max_ib_i+\min_ib_i)/{2}$, the inequality
    \begin{align}\label{eq:lm_inequality_2}
        |b_i-c|\leq \frac{1}{2}(\max_ib_i-\min_ib_i)
    \end{align}
    holds for all $i\in[n]$.
    As a consequence, 
taking absolute values
 on both sides of~\eqref{eq:lm_inequality} leads to
\begin{align}
             |(\alpha^{\mathrm{I}}-\alpha^{\mathrm{II}})^\top b|&=\Bigr|\sum_{i=1}^n(\alpha^{\mathrm{I}}_i-\alpha^{\mathrm{II}}_i)(b_i-c)\Bigr|\notag\\
          &   \leq \sum_{i=1}^n|\alpha^{\mathrm{I}}_i-\alpha^{\mathrm{II}}_i|\cdot|b_i-c|\notag\\
          &\leq \frac{1}{2}(\max_ib_i-\min_ib_i)\sum_{i=1}^n|\alpha^{\mathrm{I}}_i-\alpha^{\mathrm{II}}_i|,
\end{align}
    where the last inequality follows from~\eqref{eq:lm_inequality_2}.
\end{proof}
\fi
The union and mixture constructions of $\calW_{\mathrm G}$ enable a natural decomposition of the worst-case expectation objective.
\textcolor{black}{Lemma~\ref{lm:separability}} formalizes this separability and yields an equivalent client-wise reformulation of~\eqref{eq:problem_DRO}.
\begin{lemma}[Problem Separability]~\label{lm:separability}
The problem~\eqref{eq:problem_DRO} admits the following reformulation:
\begin{align}
    &\inf_{x\in\calX} \sup_{\mathbb{Q} \in \calW_{\mathrm G}} \mathbb{E}_{\mathbb{Q}} \left[ \ell(x,\xi) \right]\nonumber
\\
=&~ \inf_{x\in\calX}\sup_{\alpha\in\Delta_{n-1}} \sum_{i=1}^n \alpha_i 
\sup_{\bbQ_i\in\calW_i(\varepsilon_i)}
\mathbb{E}_{\mathbb{Q}_i} \left[ \ell(x,\xi) \right]. \label{eq:problem_separability}
\end{align}
\end{lemma}
\begin{proof}
\ifdefined\shortpaper\else
\textcolor{black}{For a fixed $x\in\calX$, the union structure of $\calW_{\mathrm G}$ gives}
\begin{align}\label{eq:separable_sup}
&\sup_{\bbQ\in\calW_{\mathrm G}}\ev_{\bbQ}[\ell(x,\xi)]\notag\\
&=\textcolor{black}{\sup_{\bbQ\in\bigcup_{\alpha\in\Delta_{n-1}}\calW_{\mathrm M}(\alpha)}}
\ev_{\bbQ}[\ell(x,\xi)]\notag\\
&=\sup_{\alpha\in\Delta_{n-1}}
\textcolor{black}{\sup_{\bbQ\in\calW_{\mathrm M}(\alpha)}}
\ev_{\bbQ}[\ell(x,\xi)]\textcolor{black}{.}
\end{align}
\textcolor{black}{The first equality follows from~\eqref{eq:global_ambiguity_set}. The second holds because taking the supremum over a union is equivalent to taking the supremum over its index and then over the corresponding set.}

\textcolor{black}{For any fixed $\alpha\in\Delta_{n-1}$, following the decomposition argument in~\cite[Proposition~1]{ibrahim2025fdrsvm}, we have}
\begin{align}\label{eq:fixed_alpha_separability}
&\textcolor{black}{\sup_{\bbQ\in\calW_{\mathrm M}(\alpha)}}
\ev_{\bbQ}[\ell(x,\xi)]\notag\\
&=\sup_{\{\bbQ_i\in\calW_i(\varepsilon_i)\}_{i=1}^n}
\textcolor{black}{\ev_{\sum_{i=1}^n\alpha_i\bbQ_i}[\ell(x,\xi)]}\notag\\
&=\sup_{\{\bbQ_i\in\calW_i(\varepsilon_i)\}_{i=1}^n}
\sum_{i=1}^n\alpha_i\ev_{\bbQ_i}[\ell(x,\xi)]\notag\\
&\overset{\mathrm{(a)}}{=}\sum_{i=1}^n\alpha_i
\sup_{\bbQ_i\in\calW_i(\varepsilon_i)}
\ev_{\bbQ_i}[\ell(x,\xi)]\textcolor{black}{,}
\end{align}
\textcolor{black}{where the first equality follows from the definition of $\calW_{\mathrm M}(\alpha)$ in~\eqref{eq:global_ambiguity}, and the second follows from the linearity of integration with respect to the measure. Equality (a) holds because each $\bbQ_i$ can be chosen independently within its local ambiguity set, affects only its corresponding summand, and has a nonnegative weight $\alpha_i$. Substituting~\eqref{eq:fixed_alpha_separability} into~\eqref{eq:separable_sup} and taking the infimum over $x\in\calX$ completes the proof.}
\fi
\end{proof}
\begin{theorem}[Out-of-Sample Performance]
    Let Assumption~\ref{ass:bounded_heterogeneity} hold.
   Assume that for all $x\in\calX$, the function $\xi\mapsto \ell(x,\xi)$ is $L_{\xi}$-Lipschitz.
    Under the same setting as in Proposition~\ref{prop:coverage}, with probability at least $\prod_{i=1}^n(1-\beta_i)$, we have
    \begin{align}
        \widehat{J}^\star-L_{\xi}(2r+\delta)\leq J\leq\widehat{J}^\star\textcolor{black}{,}
    \end{align}
    where $r:=\max_{i}\{\varepsilon_i\}$.
\end{theorem}
\ifdefined\shortpaper\else
\begin{proof}
    \textcolor{black}{We work on the event $E:=\bigcap_{i=1}^n\{W_p(\bbP_i,\widehat{\bbP}_i)\le\varepsilon_i\}$, on which local coverage holds simultaneously for all clients.} \textcolor{black}{Its probability} is at least $\prod_{i=1}^n(1-\beta_i)$\textcolor{black}{, as established in the proof of Proposition~\ref{prop:coverage}}. This event implies $\bbP\in\calW_{\mathrm G}$. Given the DRO optimizer $\widehat{x}^\star$, we have
\[
J=\ev_{\bbP}[\ell(\widehat{x}^\star,\xi)]
\le \sup_{\bbQ\in\calW_{\mathrm G}}\ev_{\bbQ}[\ell(\widehat{x}^\star,\xi)]
= \widehat{J}^\star,
\]
which proves the \textcolor{black}{right-hand side} inequality.

    \textcolor{black}{To bound $\widehat{J}^\star-J$, fix the observed data and $\widehat{x}^\star$. Each local ball $\calW_i(\varepsilon_i)$ contains $\widehat{\bbP}_i$ and is weakly closed in the weakly compact space $\calM(\Xi)$, since $\Xi$ is compact~\cite[Chapter~4 and Theorem~6.9]{villani2009optimal}. Thus, each local ball is nonempty and weakly compact. Since $\ell(\widehat{x}^\star,\cdot)$ is bounded and continuous, the expected-loss functional is continuous under weak convergence and attains a maximum on each local ball. Denote these local maximum values by $m_i$. By Lemma~\ref{lm:separability},}
\begin{align*}
\textcolor{black}{\sup_{\bbQ\in\calW_{\mathrm G}}\ev_{\bbQ}[\ell(\widehat{x}^\star,\xi)]}
&\textcolor{black}{=\sup_{\alpha\in\Delta_{n-1}}\sum_{i=1}^n\alpha_i m_i}\\
&\textcolor{black}{=\max_{i\in[n]}m_i.}
\end{align*}
\textcolor{black}{Choose any $j\in\operatorname*{arg\,max}_{i\in[n]}m_i$, assign unit weight to client $j$, and select its local worst-case distribution. This gives a feasible distribution attaining the global supremum. We can therefore choose worst-case weights $\widehat{\alpha}$ and a corresponding distribution $\bbQ_{\calW}$ as}
    \begin{subequations}
        \begin{align}
            \widehat{\alpha}&\textcolor{black}{\in\operatorname*{arg\,max}\limits_{\alpha\in\Delta_{n-1}}\max_{\bbQ\in\calW_{\mathrm M}(\alpha)}}
\ev_{\bbQ}[\ell(\widehat{x}^\star,\xi)]\,,\\
            \bbQ_\calW&\textcolor{black}{\in\operatorname*{arg\,max}\limits_{\bbQ\in\calW_{\mathrm M}(\widehat{\alpha})}}\mathbb{E}_{\bbQ}[\ell(\widehat{x}^\star,\xi)]\,.
        \end{align}
    \end{subequations}
\textcolor{black}{By Lemma~\ref{lm:separability}, we may choose these maximizers so that} $\bbQ_{\calW}=\sum_{i=1}^n\widehat{\alpha}_i\bbQ_{\calW}^{(i)}$,
where $\bbQ_{\calW}^{(i)}\textcolor{black}{\in\operatorname*{arg\,max}_{\bbQ_i\in\calW_i(\varepsilon_i)}}\ev_{\bbQ_i}[\ell(\widehat{x}^\star,\xi)]$ is \textcolor{black}{a} local worst-case distribution for client $i$.
Thus, the performance gap can be decomposed as
\begin{align}
        |\widehat{J}^\star-J|
        =&~|\ev_{\bbQ_\calW}[\ell(\widehat{x}^\star,\xi)]-\ev_\bbP[\ell(\widehat{x}^\star,\xi)]|\notag\\
        =&~|\sum_{i=1}^n\widehat{\alpha}_i \ev_{\bbQ_{\calW}^{(i)}}[\ell(\widehat{x}^\star,\xi)]-\sum_{i=1}^n\alpha_i^\star\ev_{\bbP_i}[\ell(\widehat{x}^\star,\xi)]|\notag\\
        =&~\Bigl|\sum_{i=1}^n\widehat{\alpha}_i \ev_{\bbQ_{\calW}^{(i)}}[\ell(\widehat{x}^\star,\xi)]-\sum_{i=1}^n\widehat{\alpha}_i\ev_{\bbP_i}[\ell(\widehat{x}^\star,\xi)]\notag\\
        &+\sum_{i=1}^n\widehat{\alpha}_i\ev_{\bbP_i}[\ell(\widehat{x}^\star,\xi)]-\sum_{i=1}^n\alpha_i^\star\ev_{\bbP_i}[\ell(\widehat{x}^\star,\xi)]\Bigr|\notag\\
        \leq&~ \underbrace{\Bigl|\sum_{i=1}^n\widehat{\alpha}_i \left(\ev_{\bbQ_{\calW}^{(i)}}[\ell(\widehat{x}^\star,\xi)]-\ev_{\bbP_i}[\ell(\widehat{x}^\star,\xi)]\right)\Bigr|}_{\text{sampling}}\notag\\
        &+\underbrace{\Bigl|\sum_{i=1}^n(\widehat{\alpha}_i-\alpha_i^\star)\ev_{\bbP_i}[\ell(\widehat{x}^\star,\xi)]\Bigr|}_{\text{ heterogeneity \& weights}}\,,\label{eq:OOS_separation}
\end{align}
where the second equality follows from arguments in Lemma~\ref{lm:separability}. 

\textcolor{black}{We next bound the two terms in~\eqref{eq:OOS_separation} on the event $E$. If $L_\xi=0$, then $\ell(\widehat{x}^\star,\cdot)$ is constant on $\Xi$, so both terms vanish and the claim follows. Suppose henceforth that $L_\xi>0$. Let $\|h\|_{\mathrm{Lip}}$ denote the smallest Lipschitz constant of $h:\Xi\to\mathbb R$.}
\textcolor{black}{Since $\ell(\widehat{x}^\star,\cdot)$ is $L_\xi$-Lipschitz, the normalized loss $\frac{1}{L_\xi}\ell(\widehat{x}^\star,\cdot)$ belongs to the class $\{h:\Xi\to\mathbb R:\|h\|_{\mathrm{Lip}}\leq1\}$. On $E$, the sampling term satisfies}
    \begin{align}\label{eq:bound1}
        &\Bigl|\sum_{i=1}^n\widehat{\alpha}_i \left(\ev_{\bbQ_{\calW}^{(i)}}[\ell(\widehat{x}^\star,\xi)]-\ev_{\bbP_i}[\ell(\widehat{x}^\star,\xi)]\right)\Bigr|\notag\\
        \leq&~\sum_{i=1}^n\widehat{\alpha}_i\Bigl|\ev_{\bbQ_{\calW}^{(i)}}[\ell(\widehat{x}^\star,\xi)]-\ev_{\bbP_i}[\ell(\widehat{x}^\star,\xi)]\Bigr|\notag\\
        \textcolor{black}{=}&~\textcolor{black}{L_\xi\sum_{i=1}^n\widehat\alpha_i\Bigl|\ev_{\bbQ_{\calW}^{(i)}}\Bigl[\frac{1}{L_\xi}\ell(\widehat{x}^\star,\xi)\Bigr]}\notag\\
        &\textcolor{black}{\qquad\qquad{}-\ev_{\bbP_i}\Bigl[\frac{1}{L_\xi}\ell(\widehat{x}^\star,\xi)\Bigr]\Bigr|}\notag\\
        \textcolor{black}{\overset{\mathrm{(a)}}{\leq}}&~ L_{\xi}\sum_{i=1}^n\widehat{\alpha}_i\sup_{\|h\|_{\mathrm{Lip}}\leq 1}\Bigl|\textcolor{black}{\ev_{\bbQ_{\calW}^{(i)}}[h(\xi)]-\ev_{\bbP_i}[h(\xi)]}\Bigr|\notag\\
        \textcolor{black}{\overset{\mathrm{(b)}}{=}}&~ L_{\xi}\sum_{i=1}^n\widehat{\alpha}_iW_1(\bbQ_{\calW}^{(i)},\bbP_i)\notag\\
        \overset{\textcolor{black}{\mathrm{(c)}}}{\leq}&~L_{\xi}\sum_{i=1}^n\widehat{\alpha}_iW_p(\bbQ_{\calW}^{(i)},\bbP_i)\notag\\
        \leq&~L_{\xi}\sum_{i=1}^n\widehat{\alpha}_i\Big(W_p(\bbQ_{\calW}^{(i)},\widehat{\bbP}_i)+W_p(\widehat{\bbP}_i,\bbP_i)\Big)\notag\\
        \textcolor{black}{\overset{\mathrm{(d)}}{\leq}}&~\textcolor{black}{2L_\xi\sum_{i=1}^n\widehat\alpha_i\varepsilon_i}\notag\\
        \leq&~ 2L_\xi r,
\end{align}
\textcolor{black}{where (a) holds because $\frac{1}{L_\xi}\ell(\widehat{x}^\star,\cdot)$ is an admissible $1$-Lipschitz test function; (b) is Kantorovich--Rubinstein duality~\cite[Theorem~3.2]{mohajerinesfahani2018datadriven}; and (c) uses $W_1\leq W_p$ for $p\geq1$~\cite[Remark~6.6]{villani2009optimal}. The next step is the triangle inequality, and (d) uses $\bbQ_{\calW}^{(i)}\in\calW_i(\varepsilon_i)$ and local coverage on $E$. The last inequality follows from $r=\max_i\varepsilon_i$ and $\widehat\alpha\in\Delta_{n-1}$.}
\textcolor{black}{For the heterogeneity term in~\eqref{eq:OOS_separation}, since} $\widehat\alpha,\alpha^{\star}\in\Delta_{n-1}$, Lemma~\ref{lm:inequality} gives
\begin{align}
        &\Bigl|\sum_{i=1}^n(\widehat{\alpha}_i-\alpha_i^\star)\ev_{\bbP_i}[\ell(\widehat{x}^\star,\xi)]\Bigr|\notag\\
        \leq&~ \frac{1}{2}(\max_i \ev_{\bbP_i}[\ell(\widehat{x}^\star,\xi)]-\min_i \ev_{\bbP_i}[\ell(\widehat{x}^\star,\xi)])\sum_{i=1}^n|\widehat{\alpha}_i-\alpha_i^\star|\notag\\
        \leq&~ (\max_i \ev_{\bbP_i}[\ell(\widehat{x}^\star,\xi)]-\min_i \ev_{\bbP_i}[\ell(\widehat{x}^\star,\xi)]).
\end{align}
Assumption~\ref{ass:bounded_heterogeneity} ensures  $\sup_{i,j} W_p(\bbP_i,\bbP_j)\le \delta$.
Applying Kantorovich--Rubinstein duality again gives
    \begin{align}
\bigl|\ev_{\bbP_i}[\ell(\widehat{x}^\star,\xi)]-\ev_{\bbP_j}[\ell(\widehat{x}^\star,\xi)]\bigr|
\le
L_\xi\,W_p(\bbP_i,\bbP_j)\leq L_{\xi}\delta,
\end{align}
which yields
\begin{align}\label{eq:bound2}
\Bigl|\sum_{i=1}^n(\widehat{\alpha}_i-\alpha_i^\star)\ev_{\bbP_i}[\ell(\widehat{x}^\star,\xi)]\Bigr|\leq L_{\xi}\delta.
\end{align}
Combining the two bounds~\eqref{eq:bound1} and~\eqref{eq:bound2} on the event $E$ completes the proof.
\end{proof}
\fi
The decomposition in~\eqref{eq:OOS_separation} separates data-dependent sampling error from structural bias. More samples reduce the former but cannot mitigate the latter. The impact of mixture-weight deviations grows with the heterogeneity level $\delta$ and vanishes when $\delta=0$. Consequently, data collection is most beneficial when the sampling term dominates. Otherwise, one should focus on reducing heterogeneity or tightening the feasible weight set.
\section{Tractable Reformulation and Algorithm}
\begin{algorithm}[t]
\caption{BiDRO-FL}
\label{alg:dorfl}
\begin{algorithmic}[1]
\REQUIRE Sampling distributions $\{\widehat{\mathbb{P}}_i\}_{i=1}^n$, the constraint set $\calX$, step sizes $\eta_x$, $\eta_\alpha$ and initial points $x_1\in\calX$, $\alpha_1\in\Delta_{n-1}$
\FOR{$t = 1,\ldots,T$}
\STATE Server broadcasts $x_t$ and $\alpha_t$ to each client
    \FOR{client $i = 1,\ldots,n$}
        \STATE Sample $\widehat{\zeta}_{i,t}$ from $\widehat{\mathbb{P}}_i$
        \STATE Find an $\epsilon$-approximate maximizer $z_{i,t}$ of~\eqref{eq:inner_sup_xi}
        \STATE Construct gradient estimate $g^{x}_{i,t}$ by~\eqref{eq:gradient_x}
        \STATE Construct gradient estimate $g^{\alpha}_{i,t}$ by~\eqref{eq:gradient_alpha}
        \STATE Update $x_{i,t+1}$ by~\eqref{eq:local_update}
        \STATE Client $i$ sends $x_{i,t+1}$, $g^{\alpha}_{i,t}$ back to the server
    \ENDFOR
    \STATE Server updates $x_{t+1}$ by~\eqref{eq:global_x_update} and $\alpha_{t+1}$ by~\eqref{eq:global_alpha_update}
\ENDFOR
\end{algorithmic}
\end{algorithm}
The separable formulation~\eqref{eq:problem_separability} still requires optimization over local probability measures. We replace each hard Wasserstein-ball constraint with a penalty and then derive an equivalent empirical representation of the resulting objective:

\begin{small}
    \begin{align}~\label{eq:penalized_DRO}
    \inf_{x\in\calX} \! \sup_{\alpha\in\Delta_{n-1}} \! \sum_{i=1}^n \alpha_i \!
\sup_{\mathbb{Q}_i\in\calM(\Xi)} \!
\mathbb{E}_{\mathbb{Q}_i} \Big[ \ell(x,\xi) \!-\!\rho_i \Big(W_p(\bbQ_i,\widehat{\bbP}_i)\Big)^p \Big]
\end{align}
\end{small}%
Here $\rho_i>0$ is a client-specific robustness penalty. \textcolor{black}{The following reformulation specializes~\cite[Proposition~1]{sinha2017certifying} to the client-wise empirical distributions.}
\begin{proposition}[\textcolor{black}{Finite-dimensional reformulation}]\label{prop:penalty_reformulation}
Suppose $\ell(x,\cdot)$ is continuous on $\Xi$ for every $x\in\calX$, and $c$ is continuous on $\Xi\times\Xi$. \textcolor{black}{Then, problem~\eqref{eq:penalized_DRO} is equivalent to}
    \begin{align}~\label{eq:reformed_DRO}
    \inf_{x\in\calX} \! \sup_{\alpha\in\Delta_{n-1}} \! \sum_{i=1}^n \alpha_i
\ev_{\zeta\sim\widehat{\bbP}_i} \! \Big[\sup_{\xi\in\Xi}\phi_i(x,\zeta,\xi)\Big],
\end{align}
where $\phi_i(x,\zeta,\xi):=\ell(x,\xi)-\rho_i c(\zeta,\xi)$ with the transportation cost function $c(\cdot,\cdot)$ used in Definition~\ref{def:Wp-metric}.
\end{proposition}
\begingroup

\ifdefined\shortpaper\else
\begin{proof}
Fix $x\in\calX$ and $i\in[n]$. \textcolor{black}{We rewrite the local penalized objective using transport plans.} Denote the local penalized value by
\[
L_i(x):=\sup_{\bbQ_i\in\calM(\Xi)}
\left\{\ev_{\bbQ_i}[\ell(x,\xi)]-\rho_i W_p^p(\bbQ_i,\widehat{\bbP}_i)\right\}.
\]
\emph{Step 1: Transport-plan reformulation.} By symmetry of $W_p$ and Definition~\ref{def:Wp-metric},
\begin{align*}
L_i(x)
&=\sup_{\substack{\bbQ_i\in\calM(\Xi)\\
\gamma\in\Gamma(\widehat{\bbP}_i,\bbQ_i)}}
\int_{\textcolor{black}{\Xi\times\Xi}}\phi_i(x,\zeta,\xi)\,\mathrm d\gamma\\
&=\sup_{\gamma:\,\gamma_\zeta=\widehat{\bbP}_i}
\int_{\textcolor{black}{\Xi\times\Xi}}\phi_i(x,\zeta,\xi)\,\mathrm d\gamma\,,
\end{align*}
\textcolor{black}{where $\gamma_\zeta$ denotes the marginal of $\gamma$ with respect to $\zeta$.} The first equality uses $\rho_i>0$ to turn the infimum of transport costs into a supremum of their negatives. The second follows because the $\xi$-marginal is free when $\bbQ_i$ ranges over $\calM(\Xi)$.
\emph{Step 2: Empirical decomposition.} By~\eqref{eq:empirical_distri}, every feasible coupling admits the representation
\[
\gamma(\mathrm d\zeta,\mathrm d\xi)
=\frac{1}{N_i}\sum_{k=1}^{N_i}
\delta_{\widehat\xi_i^k}(\mathrm d\zeta)\nu_k(\mathrm d\xi),
\qquad \nu_k\in\calM(\Xi),
\]
where $\nu_k$ describes the destination distribution of the mass starting from sample $\widehat\xi_i^k$. Every such collection defines a feasible coupling. For repeated empirical atoms, one may use the same conditional measure for each occurrence to represent any given coupling. Thus
\begingroup\color{black}
\begin{align}\label{eq:empirical_coupling_value}
L_i(x)
&=\sup_{\nu_1,\ldots,\nu_{N_i}\in\calM(\Xi)}
\frac{1}{N_i}\sum_{k=1}^{N_i}\int_{\Xi}\phi_i(x,\widehat\xi_i^k,\xi)\,\mathrm d\nu_k.
\end{align}
\endgroup
The distributions $\nu_k$ can be chosen independently, and each appears only in its corresponding summand.\\
\emph{Step 3: Pointwise upper bound.} For every $\nu_k\in\calM(\Xi)$,
\[
\int_{\textcolor{black}{\Xi}}\phi_i(x,\widehat\xi_i^k,\xi)\,\mathrm d\nu_k
\le\max_{\xi\in\Xi}\phi_i(x,\widehat\xi_i^k,\xi).
\]
Consequently,
\[
L_i(x)\le\frac{1}{N_i}\sum_{k=1}^{N_i}
\max_{\xi\in\Xi}\phi_i(x,\widehat\xi_i^k,\xi).
\]
\noindent\emph{Step 4: Attainment of the bound.} \textcolor{black}{Compactness of $\Xi$ and continuity of $\xi\mapsto\phi_i(x,\widehat\xi_i^k,\xi)$ ensure a maximizer $\xi_k^\star$ for each summand. Choosing $\nu_k=\delta_{\xi_k^\star}$ in~\eqref{eq:empirical_coupling_value} implies that the upper bound derived in Step 3 is attained. Therefore,}
\[
L_i(x)=\ev_{\zeta\sim\widehat{\bbP}_i}
[\sup_{\xi\in\Xi}\phi_i(x,\zeta,\xi)].
\]
This identity holds for every $x$ and $i$. Substituting it into~\eqref{eq:penalized_DRO} proves~\eqref{eq:reformed_DRO}.
\end{proof}
\fi
\endgroup
We solve~\eqref{eq:reformed_DRO} using BiDRO-FL (Algorithm~\ref{alg:dorfl}), where ``Bi'' refers to the two sources of uncertainty.
\textcolor{black}{At each round $t$, each client $i$ draws a sample
$\widehat{\zeta}_{i,t}\sim\widehat{\bbP}_i$ from its local dataset,
independently of the sampling history, and uses it to compute
a gradient step. Specifically, client $i$ solves the local problem:}
\begin{align}\label{eq:inner_sup_xi}
    \arg\sup_{\xi\in\Xi} \; \phi_i(x_t,\widehat{\zeta}_{i,t},\xi)
\end{align}
to obtain an $\epsilon$-approximate maximizer $z_{i,t}$ satisfying $\operatorname{dist}(z_{i,t},\calZ^\star_{i,t})\leq\epsilon$. Here, $\operatorname{dist}$ denotes the distance from a point to the set, and $\calZ^\star_{i,t}$ denotes the set of exact maximizers \textcolor{black}{of~\eqref{eq:inner_sup_xi}}.
Using $z_{i,t}$, each client $i$ \textcolor{black}{computes} stochastic gradients:
\begin{subequations}
    \begin{align}
        g_{i,t}^x&=\nabla_{x}\ell(x_t,z_{i,t})\,,\label{eq:gradient_x}\\
        g_{i,t}^{\alpha}&=\ell(x_t,z_{i,t})-\rho_i c(\widehat{\zeta}_{i,t},z_{i,t}).\label{eq:gradient_alpha}
    \end{align}
\end{subequations}
\textcolor{black}{Subsequently, client $i$ takes the gradient-descent step}
\begin{align}\label{eq:local_update}
    x_{i,t+1}=x_t-\eta_xg_{i,t}^x\,,
\end{align}
with step size $\eta_x$.
\textcolor{black}{Then, client $i$ broadcasts $x_{i,t+1}$ and $g_{i,t}^{\alpha}$ to the server.}
The server updates the global parameters $(x_t,\alpha_t)$:
\begin{subequations}
    \begin{align}
        x_{t+1}&=\operatorname{Proj}_{\calX}\Big(\sum_{i=1}^n\alpha_{i,t}x_{i,t+1}\Big)\,,\label{eq:global_x_update}\\
        \alpha_{t+1}&=\operatorname{Proj}_{\Delta_{n-1}}\left(\alpha_t+\eta_\alpha g_t^{\alpha}\right)\,, \label{eq:global_alpha_update}
    \end{align}
\end{subequations}
where $g_t^\alpha:=(g_{1,t}^{\alpha},\ldots,g_{n,t}^{\alpha})$, $\eta_\alpha$ is the step size, and  $\mathrm{Proj}_{\calY}(\textcolor{black}{\cdot})$ denotes the Euclidean projection onto $\calY$.
Updated parameters $(x_{t+1},\alpha_{t+1})$ are then broadcast to all clients.
\newcommand{\simulationfigure}{%
\begin{figure*}[t]
  \centering
  \begin{subfigure}{0.49\textwidth}
    \centering
    \includegraphics[width=\linewidth]{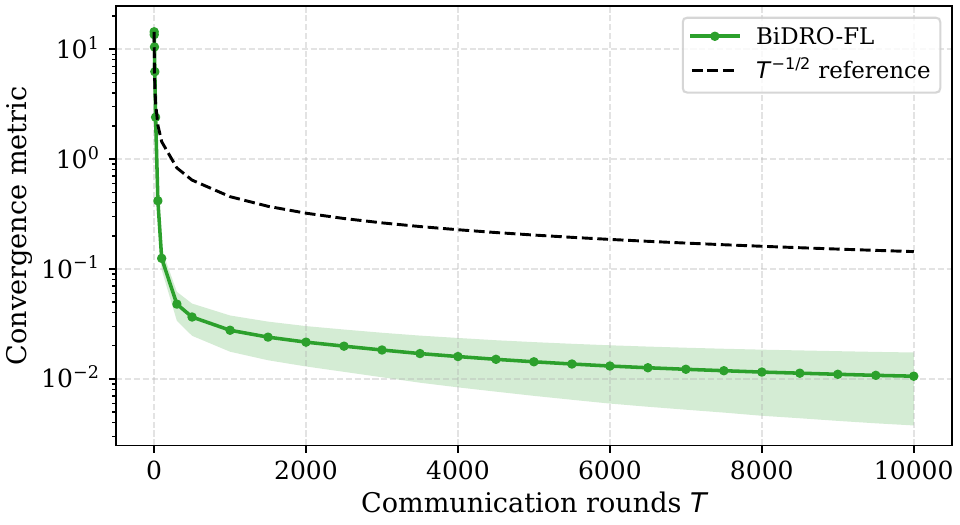}
    \caption{} %
    \label{fig:convergence}
  \end{subfigure}\hfill
  \begin{subfigure}{0.49\textwidth}
    \centering
    \includegraphics[width=\linewidth]{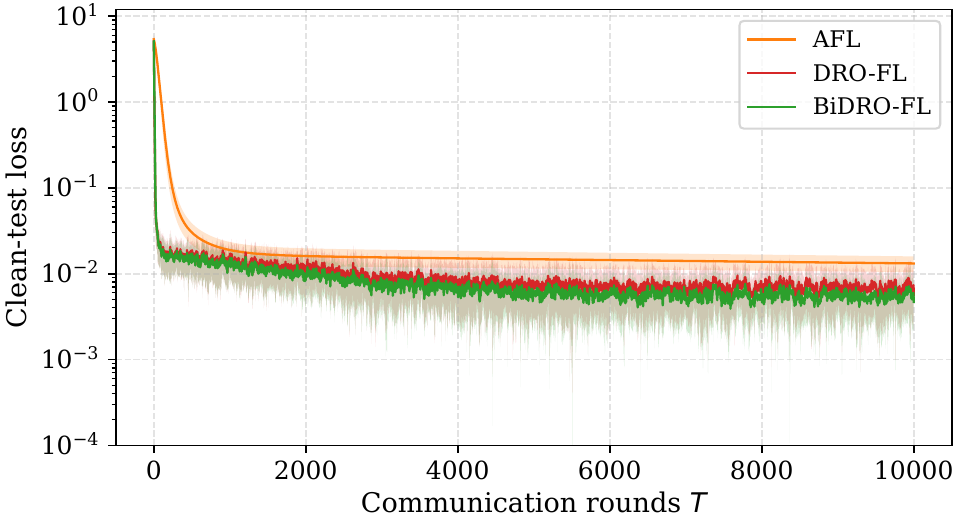}
    \caption{} %
    \label{fig:OOS}
  \end{subfigure}
  \caption{(a) Empirical convergence metric. (b) Clean-test loss comparison. Curves show means over $20$ runs; shading denotes $\pm$ one standard deviation across runs.}
  \label{fig:simulation}
\end{figure*}%
}
\section{Convergence Analysis}
\textcolor{black}{For Algorithm~\ref{alg:dorfl}, define}
\begin{align}
F(x,\alpha):=\sum_{i=1}^n\alpha_i\ev_{\zeta\sim\widehat{\bbP}_i}[f_i(x,\zeta)]\,,
\end{align}
where $f_i(x,\zeta)\!:=\sup_{\xi\in\Xi}~ \phi_i(x,\zeta,\xi)$. \textcolor{black}{Note that the reformulated problem~\eqref{eq:reformed_DRO} for which we designed the algorithm is $\inf_{x \in \mathcal{X}}\sup_{\alpha\in\Delta_{n-1}}F(x,\alpha)$.}
We begin by stating the following assumptions.
\begin{assumption}\label{ass:main}
The regularity and geometry conditions are provided as follows.
\begin{enumerate}[label=(\roman*)]
    \item The map $(x,\xi)\mapsto\ell(x,\xi)$ is differentiable on $\calX\times\Xi$. For any $\xi\in\Xi$, the map $x\mapsto\ell(x,\xi)$ is convex on $\calX$ \textcolor{black}{and} it satisfies the bounds  $|\ell(x,\xi)|\leq B_1$ and $\|\nabla_x\ell(x,\xi)\|\le B_2$ for all $x \in \calX$. 
    Finally, for any $x\in\calX$, the map $\xi\mapsto \nabla_x\ell(x,\xi)$ is $L_{x\xi}$-Lipschitz on $\Xi$.~\label{ass:1}
    \item The map $(\zeta,\xi)\mapsto c(\zeta,\xi)$ is continuous on $\Xi\times\Xi$.
    \item For any $i\in[n]$ and $(x,\zeta)\in\calX\times\Xi$, the map $\xi\mapsto \ell(x,\xi)-\rho_i c(\zeta,\xi)$ is $L_\xi^c$-Lipschitz continuous.
\end{enumerate}
\end{assumption}
We now establish the convergence result of Algorithm~\ref{alg:dorfl}.
\begin{theorem}[\textcolor{black}{Convergence of BiDRO-FL}]\label{thm:convergence}
    Let Assumption~\ref{ass:main} hold. With $\overline{x}_T:=\frac{1}{T}\sum_{t=1}^Tx_t$, selecting $\eta_x=\eta_\alpha=1/\sqrt{T}$ yields
    \begin{small}
        \begin{align}
        \ev\Big[\sup\limits_{\alpha\in\Delta_{n-1}}F(\overline{x}_T,\alpha)-\inf\limits_{x\in\calX}\sup\limits_{\alpha\in\Delta_{n-1}}F(x,\alpha)\Big]=\calO (T^{-1/2}+\epsilon).
    \end{align}
    \end{small}%
\end{theorem}

\ifdefined\shortpaper\fi
\begingroup
\ifdefined\shortpaper\else
\begin{proof}
By the convexity-concavity of $F$ and the minimax theorem, the duality gap is bounded by
        \begin{align}
        &~\sup\limits_{\alpha\in\Delta_{n-1}}F(\overline{x}_T,\alpha)-\inf\limits_{x\in\calX}\sup\limits_{\alpha\in\Delta_{n-1}}F(x,\alpha)\nonumber\\
        \textcolor{black}{\overset{\mathrm{(a)}}{=}}&~\sup\limits_{\alpha\in\Delta_{n-1}}F(\overline{x}_T,\alpha)-\sup\limits_{\alpha\in\Delta_{n-1}}\inf\limits_{x\in\calX}F(x,\alpha)\nonumber\\
        \leq &~\sup\limits_{\alpha\in\Delta_{n-1}}F(\overline{x}_T,\alpha)-\inf\limits_{x\in\calX}F(x,\overline{\alpha})\nonumber\\
        \textcolor{black}{\overset{\mathrm{(b)}}{=}}&~\sup\limits_{\substack{x\in\calX\\ \alpha\in\Delta_{n-1}}}\Big\{F(\overline{x}_T,\alpha)-\frac{1}{T}\sum_{t=1}^TF(x,\alpha_t)\Big\}\nonumber\\
        \textcolor{black}{\overset{\mathrm{(c)}}{\leq}}&~\sup\limits_{\substack{x\in\calX\\ \alpha\in\Delta_{n-1}}}\Big\{\frac{1}{T}\sum_{t=1}^TF(x_t,\alpha)-\frac{1}{T}\sum_{t=1}^TF(x,\alpha_t)\Big\}\nonumber\\
        =&~\frac{1}{T}\sup\limits_{\substack{x\in\calX\\ \alpha\in\Delta_{n-1}}}\Big\{\sum_{t=1}^T\left(F(x_t,\alpha)-F(x,\alpha_t)\right)\Big\}\,,~\label{eq:separable_performance_gap}
    \end{align}
    where $\overline{\alpha}=\frac{1}{T}\sum_{t=1}^T\alpha_t$. \textcolor{black}{Since $\calX$ and $\Delta_{n-1}$ are compact and convex, and $F$ is continuous and convex--concave, Sion\textquotesingle{}s minimax theorem~\cite[Theorem~3.4]{sion1958general} justifies interchanging the infimum and supremum in (a). Equality (b) follows from the linearity of $F(x,\cdot)$, while inequality (c) follows from the convexity of $F(\cdot,\alpha)$ and Jensen\textquotesingle{}s inequality.}
   For a given $\alpha$, since the map $x\mapsto F(x,\alpha)$ is not necessarily differentiable, let $\partial_xF(x,\alpha)$ be the subdifferential of $F(\cdot,\alpha)$ at $x$.
    Moreover, the convexity-concavity of $F$ results in
        \begin{align}
            &~F(x_t,\alpha)-F(x,\alpha_t)\nonumber\\
            =&~F(x_t,\alpha)-F(x_t,\alpha_t)+F(x_t,\alpha_t)-F(x,\alpha_t)\nonumber\\
            \leq&~\langle \alpha-\alpha_t,g_t^{\alpha}\rangle+\langle \alpha-\alpha_t,\nabla_\alpha F(x_t,\alpha_t)-g_t^{\alpha}\rangle\nonumber\\
            &~+\langle x_t-x,g_t^x\rangle+\langle x_t-x,sg_t^x-g_t^x\rangle\,,~\label{eq:F_gap}
        \end{align}
        where $g_t^x=\sum_{i=1}^n\alpha_{i,t}g_{i,t}^x$ and $sg_t^x\in\partial_x F(x_t,\alpha_t)$. Substituting~\eqref{eq:F_gap} into~\eqref{eq:separable_performance_gap}, we obtain
        \begin{small}
            \begin{align}
            &~\sup\limits_{\alpha\in\Delta_{n-1}}F(\overline{x}_T,\alpha)-\inf\limits_{x\in\calX}\sup\limits_{\alpha\in\Delta_{n-1}}F(x,\alpha)\notag\\
        \leq&~\frac{1}{T}\Big(\sup_{\alpha\in\Delta_{n-1}}\sum_{t=1}^T\langle \alpha-\alpha_t,g_t^{\alpha}\rangle+\sup_{x\in\calX}\sum_{t=1}^T\langle x_t-x,g_t^x\rangle\notag\\
        &~+\sup_{\alpha\in\Delta_{n-1}}\sum_{t=1}^T\langle \alpha-\alpha_t,\nabla_\alpha F(x_t,\alpha_t)-g_t^{\alpha}\rangle\notag\\~
        &~+\sup\limits_{x\in\calX}\sum_{t=1}^T \langle x_t-x,sg_t^x-g_t^x\rangle\Big).\label{eq:performance_gap_4terms}
        \end{align}
        \end{small}%
\fi

\textcolor{black}{The four terms in~\eqref{eq:performance_gap_4terms} represent the primal and dual updates and their estimation errors.}
\ifdefined\shortpaper\else
\textcolor{black}{We bound them using the following four lemmas, whose proofs are deferred to Appendix~\ref{app:convergence_proofs}.}
\fi
\textcolor{black}{Throughout these lemmas, Assumption~\ref{ass:main} holds and the iterates are generated by Algorithm~\ref{alg:dorfl}; $sg_t^x$ is the exact-oracle subgradient constructed in the proof of Lemma~\ref{lem:primal_estimation_error}.}

\noindent\textcolor{black}{\textit{Primal-side bounds.}}
\textcolor{black}{We first bound the primal update term and its estimation error.}

\begin{samepage}
\begin{lemma}[Primal update bound]\label{lem:primal_update_bound}
\textcolor{black}{Almost surely, for} every $x\in\calX$, the primal updates satisfy
\begin{align}\label{eq:bound_1}
\sum_{t=1}^T\langle x_t-x,g_t^x\rangle
\leq \frac{D_x^2}{2\eta_x}+\frac{\eta_xTB_2^2}{2}.
\end{align}
\end{lemma}
\end{samepage}

\textcolor{black}{The next lemma accounts for stochastic sampling and inexact inner maximization.}

\begin{samepage}
\begin{lemma}[Primal estimation error]\label{lem:primal_estimation_error}
The primal gradient-estimation error satisfies
\begin{align}\label{eq:bound_2}
&\ev\Big[\sup_{x\in\calX}\sum_{t=1}^T
\langle x_t-x,sg_t^x-g_t^x\rangle\Big]\notag\\
&\qquad\leq 2D_xB_2\sqrt{T}+TD_xL_{x\xi}\epsilon.
\end{align}
\end{lemma}
\end{samepage}

\noindent\textcolor{black}{\textit{Dual-side bounds.}}
\textcolor{black}{We now bound the corresponding terms for the mixture weights.} Set $D_\alpha:=\max_{\alpha,\alpha'\in\Delta_{n-1}}\|\alpha-\alpha'\|$. By continuity of $c$ and compactness of $\Xi$, choose $D_c>0$ such that $|c(\zeta,\xi)|\leq D_c$ for all $\zeta,\xi\in\Xi$, and define $C_i:=B_1+\rho_iD_c$.

\begin{samepage}
\begin{lemma}[Dual update bound]\label{lem:dual_update_bound}
\textcolor{black}{Almost surely, for} every $\alpha\in\Delta_{n-1}$, the dual updates satisfy
\begin{align}\label{eq:bound_3}
\sum_{t=1}^T\langle\alpha-\alpha_t,g_t^\alpha\rangle
\leq\frac{D_\alpha^2}{2\eta_\alpha}
+\frac{\eta_\alpha T}{2}\sum_{i=1}^n C_i^2.
\end{align}
\end{lemma}
\end{samepage}

\textcolor{black}{The final lemma bounds the dual estimation error.}

\begin{samepage}
\begin{lemma}[Dual estimation error]\label{lem:dual_estimation_error}
The dual gradient-estimation error satisfies
\begin{align}\label{eq:bound_4}
&\ev\Big[\sup_{\alpha\in\Delta_{n-1}}\sum_{t=1}^T
\langle\alpha-\alpha_t,\nabla_\alpha F(x_t,\alpha_t)-g_t^\alpha\rangle\Big]\notag\\
&\qquad\leq 2\sqrt{T}\sum_{i=1}^n C_i+nTL_\xi^c\epsilon.
\end{align}
\end{lemma}
\end{samepage}

Taking expectations in~\eqref{eq:performance_gap_4terms} and applying Lemmas~\ref{lem:primal_update_bound}--\ref{lem:dual_estimation_error}, we obtain
\ifdefined\shortpaper\else
    \begin{align}
&~\ev\Big[\sup\limits_{\alpha\in\Delta_{n-1}}F(\overline{x}_T,\alpha)-\inf\limits_{x\in\calX}\sup\limits_{\alpha\in\Delta_{n-1}}F(x,\alpha)\Big]\notag\\
        \leq&~ \frac{D_x^2}{2\eta_x T}+\frac{\eta_xB_2^2}{2}+\frac{2D_xB_2}{\sqrt{T}}+D_xL_{x\xi}\epsilon\notag\\
    &~+\frac{D_\alpha^{2}}{2\eta_\alpha T}+\frac{\eta_\alpha }{2}\sum_{i=1}^nC_i^2+\frac{2}{\sqrt{T}}\sum_{i=1}^nC_i+nL_{\xi}^c\epsilon\notag\\
    =&~\frac{2D_xB_2+2\sum_{i=1}^nC_i}{\sqrt{T}}+\frac{D_x^2+D_\alpha^{2}+B_2^2+\sum_{i=1}^nC_i^2}{2\sqrt{T}}\notag\\
    &~+(D_xL_{x\xi}+nL_{\xi}^c)\epsilon
\end{align}
where the last equality holds when selecting $\eta_x=\eta_\alpha=\frac{1}{\sqrt{T}}$.
\fi
\end{proof}
\endgroup

\ifdefined\shortpaper\else\simulationfigure\fi

\section{Simulation}\label{sec:simulation}
\begingroup
\ifdefined\shortpaper\else
\textcolor{black}{We consider a synthetic federated linear regression setting with $n$ heterogeneous clients, learning a shared linear predictor from corrupted observations and testing on clean data.} \textcolor{black}{The shared ground-truth coefficient vector $a\in\bbR^3$ is unknown to the clients, and $x\in\bbR^3$ is the parameter vector of the predictor learned from their observed data.}
\textcolor{black}{The clean feature vector $\theta\in\bbR^3$ represents the features before observation errors are added.} For client $\textcolor{black}{i\in[n]}$, clean features \textcolor{black}{$\theta$} follow $\calN(\textcolor{black}{\mu_i},I_3)$ restricted to $[-500,500]^3$, where $I_3$ is the identity matrix\textcolor{black}{. The Gaussian mean vector is $\mu_i=\delta_i v_i$, where $\delta_i\geq0$ controls its magnitude and $v_i$ is a unit direction.}
The response is $y=\theta^\top a$. \textcolor{black}{Client $i$ observes $(\bar\theta,y)$, with} $\textcolor{black}{\bar\theta}=\theta+s_iv_i+\gamma$, \textcolor{black}{where $s_i$ is a prescribed shift coefficient along $v_i$ and} $\gamma\sim\calN(0,0.1^2I_3)$ \textcolor{black}{is observation noise. Thus, the label is generated from the clean feature vector, while training uses the corrupted feature vector.}
\fi
For $\xi=(\theta,y)$, we use $\ell(x,\xi)=(y-\theta^\top x)^2$ and $c(\zeta,\xi)=\|\zeta-\xi\|^2$. Applied to empirical samples $(\textcolor{black}{\bar\theta},y)$, formulation~\eqref{eq:penalized_DRO} optimizes client weights and permits perturbations of both features and labels.
\ifdefined\shortpaper\else
\par
\textcolor{black}{We use $n=5$ clients with sample sizes~$(30,40,80,150,300)$.} We set $a=[1,0.5,-0.5]^\top$ and $(\delta_i)=(3,1,1,0.5,0.5)$. The shifts are $(s_i)=(0.2,-0.1,0,0,0)$.
\textcolor{black}{The unit directions are constructed so that $v_1,v_3$ are nearly aligned with $a$, $v_2$ is nearly aligned with $-a$, and $v_4,v_5$ are orthogonal to $a$.}
\textcolor{black}{The directions are generated with this structure once per run as follows. First, let $u=a/\|a\|$. Draw two independent standard Gaussian vectors in $\bbR^3$, project each onto $u^\perp$, and normalize the projections to obtain $b_1,b_2$. The directions $v_1,v_2,v_3$ are the normalized versions of $u+0.2b_1$, $-u+0.2b_1$, and $u+0.2b_2$, respectively; $v_4=b_1$ and $v_5=b_2$. Thus, the $v_i$ are random through $b_1,b_2$ and are linked by this construction.}
\textcolor{black}{In this synthetic setting, each client uses a shared direction $v_i$ for its Gaussian mean and systematic observation shift.}
Corrupted features and labels are clipped to $[-500,500]$ coordinatewise. All methods initialize the model at zero and the mixture weights at the empirical proportions $N_i/\sum_jN_j$, with $x\in[-100,100]^3$.
\fi
\ifdefined\shortpaper\fi

All methods run for $10^4$ communication rounds. BiDRO-FL uses one sample per client per round and client-specific penalties.
\ifdefined\shortpaper\else
Its step sizes are $\eta_x=\eta_\alpha=0.01$, with penalties $\rho=(15,10,5,4,2.5)$.
\fi
\ifdefined\shortpaper\else
The inner solver performs $30$ projected ascent steps of size $0.003$, followed by a random perturbation of norm $0.01$ and projection.
\fi
AFL~\cite{mohri2019agnostica} uses full local empirical losses.
\ifdefined\shortpaper\else
Its step sizes are $0.001$ for both variables.
\fi The DRO-FL curve is a shared-penalty BiDRO-FL variant with $\rho_i=5$.
\endgroup
\begingroup
We use $20$ runs, sharing training data across methods.
\ifdefined\shortpaper\else
The data seeds are $2026$--$2045$.
\fi Within each run, all methods are evaluated on a common clean test mixture using $4000$ samples.
\ifdefined\shortpaper\else
The test mixture weights are obtained from the final BiDRO-FL iterate and determine the allocation of test samples across clients.
\fi Figure~\ref{fig:convergence} shows the decrease of the numerically evaluated robust objective relative to a common reference value, with $T^{-1/2}$ included for comparison.
\ifdefined\shortpaper\else
The common reference value is obtained by numerically solving the empirical minimax problem on the first run's dataset and is reused across runs.
\fi Figure~\ref{fig:OOS} evaluates prediction accuracy on clean test data under the common test mixture within each run. The final mean MSE is $0.00470$ for BiDRO-FL, compared with $0.00569$ for DRO-FL and $0.01318$ for AFL. These results demonstrate that BiDRO-FL achieves better prediction accuracy than the compared methods.

\endgroup

\section{Conclusion}\label{sec:conclusion}
In this paper, we study a federated learning problem under two sources of uncertainty: within-client distributional ambiguity and cross-client mixture uncertainty.
We construct a global ambiguity set by mixing local ambiguity sets, and derive a high-probability out-of-sample bound in terms of sample sizes, client heterogeneity, and mixture-weight deviations.
To solve the resulting DRO problem efficiently, we introduce a penalty-based relaxation of the set-membership constraints and develop the BiDRO-FL algorithm with a provable convergence rate.
\textcolor{black}{Future work would include exploring federated and distributed algorithms that are communication efficient and do not require a central server for coordination.}

\ifdefined\shortpaper\else
\begingroup
\useRomanappendicesfalse
\appendices
\section{\textcolor{black}{Proofs of Lemmas~\ref{lem:primal_update_bound}--\ref{lem:dual_estimation_error}}}\label{app:convergence_proofs}

\textcolor{black}{Throughout this appendix, we condition on the observed local datasets. Let $\mathcal F_t$ denote the $\sigma$-algebra generated by the algorithmic history before sampling at round $t$, including the initial iterates. Then $x_t$ and $\alpha_t$ are $\mathcal F_t$-measurable. By fresh sampling, the conditional distribution of $\widehat\zeta_{i,t}$ given $\mathcal F_t$ is $\widehat{\bbP}_i$.}

\subsection{\textcolor{black}{Proof of Lemma~\ref{lem:primal_update_bound}}}\label{app:primal_update_bound}
\begin{proof}

By the nonexpansiveness of the projection in~\eqref{eq:global_x_update} and the local update~\eqref{eq:local_update}, we have
        \begin{align}~\label{eq:upperbound_xg}
            &~\|x_{t+1}-x\|^2\nonumber\\
            \leq&~\|x_t-\sum_{i=1}^n\alpha_{i,t}\eta_xg_{i,t}^x -x\|^2\nonumber\\
            =&~\|x_t-x\|^2+\eta_x^2\|\sum_{i=1}^n\alpha_{i,t}g_{i,t}^x\|^2-2\eta_x\langle x_t-x,g_t^x\rangle\,\textcolor{black}{.}
        \end{align}
        Summing~\eqref{eq:upperbound_xg} over $T$ rounds yields the following upper bound:
            \begin{small}
                \begin{align}
                &~\sum_{t=1}^T\langle x_t-x,g_t^x\rangle\notag\\
                \leq&~\frac{1}{2\eta_x}\sum_{t=1}^T\Big(\|x_t-x\|^2-\|x_{t+1}-x\|^2+\eta_x^2\Big\|\sum_{i=1}^n\alpha_{i,t}g_{i,t}^x\Big\|^2\Big)\notag\\
                \leq&~\frac{1}{2\eta_x}\|x_1-x\|^2+\frac{\eta_x}{2}\sum_{t=1}^T\Big\|\sum_{i=1}^n\alpha_{i,t}g_{i,t}^x\Big\|^2\notag\\
                \leq&~\frac{D_x^2}{2\eta_x}+\frac{\eta_x}{2}\sum_{t=1}^T\Big(\sum_{i=1}^n\alpha_{i,t}\|g_{i,t}^x\|\Big)^2\notag\\
                \leq&~\frac{D_x^2}{2\eta_x}+\frac{\eta_xTB_2^2}{2}\,,\notag
            \end{align}\end{small}%
        where the third inequality holds due to the boundedness of the set $\calX$. The last inequality follows from the bounded gradients of $\ell$ with respect to $x$ \textcolor{black}{(see Assumption~\ref{ass:main})}.

\end{proof}

\subsection{\textcolor{black}{Proof of Lemma~\ref{lem:primal_estimation_error}}}\label{app:primal_estimation_error}
\begin{proof}

By Assumption~\ref{ass:main}, the map $x\mapsto f_i(x,\zeta)$ is convex for each $i$ and $\zeta$. Define the set of all inner maximizers by $\calZ_i^\star(x,\zeta):=\arg\max_{\xi\in\Xi}\phi_i(x,\zeta,\xi)$. Danskin's theorem~\cite{Bertsekas2016Nonlinear} gives
        \begin{align}
         \partial_x f_i(x,\zeta)=\operatorname{conv}\bigl\{\nabla_x\ell(x,\xi):\xi\in\calZ_i^\star(x,\zeta)\bigr\}.
        \end{align}
        Then, the subgradient of $F(x,\alpha)$ with respect to $x$ can be computed as
        \begin{align}
            \partial_x F(x_t,\alpha_t)=\sum_{i=1}^n \alpha_{i,t}\ev_{\zeta\sim\widehat{\bbP}_i}\left[\partial_x f_i(x_t,\zeta)\right].
        \end{align}
        For the implemented approximate maximizer $z_{i,t}$, choose a nearest point $z_{i,t}^{\mathrm{opt}}$ in $\calZ_{i,t}^\star=\calZ_i^\star(x_t,\widehat\zeta_{i,t})$. Such a point exists because this set is nonempty and compact. The oracle accuracy condition then gives $\|z_{i,t}-z_{i,t}^{\mathrm{opt}}\|=\operatorname{dist}(z_{i,t},\calZ_{i,t}^\star)\leq\epsilon$.
        Consequently, the specific maximizer $z_{i,t}^{\mathrm{opt}}$ yields the following subgradients:
        \begin{subequations}
            \begin{align}
                &G_{i,t}^x(\widehat{\zeta}_{i,t}):=\nabla_x\ell(x_t,z_{i,t}^{\mathrm{opt}})\in \partial_x f_i(x_t,\widehat\zeta_{i,t}),\\
                &sg_t^x=\ev\Big[\sum_{i=1}^n \alpha_{i,t}G_{i,t}^x(\widehat{\zeta}_{i,t})\textcolor{black}{\,\Big|\,\mathcal F_t}\Big]\in \partial_x F(x_t,\alpha_t),\label{eq:sg_expectation}
            \end{align}
        \end{subequations}
        where  $G_{i,t}^x(\widehat{\zeta}_{i,t})$ is an exact subgradient of $f_i(\cdot,\widehat{\zeta}_{i,t})$.
        The $\epsilon$-approximate maximizer induces a biased error $e_{i,t}^x:=g_{i,t}^x-G_{i,t}^x(\widehat{\zeta}_{i,t})$, bounded by
        \begin{align}
            \begin{aligned}
                \|e_{i,t}^x\|&=\|\nabla_x \ell(x_t,z_{i,t})-\nabla_x\ell(x_t,z_{i,t}^{\mathrm{opt}})\|\\
                &\leq L_{x\xi}\|z_{i,t}-z_{i,t}^{\mathrm{opt}}\|\leq  L_{x\xi}\epsilon\,,
            \end{aligned}
        \end{align}
        where the first inequality follows from the $L_{x\xi}
        $-Lipschitz continuity of $\nabla_x\ell(x,\xi)$ with respect to $\xi$.
        It follows that
        \begin{small}
            \begin{align}
            &~\sup_{x\in\calX}\sum_{t=1}^T\langle x_t-x,sg_t^x-g_t^x\rangle\notag\\
            =&~\sup_{x\in\calX}\Bigl\{\sum_{t=1}^T\langle x_t-x,sg_t^x-\sum_{i=1}^n\alpha_{i,t}G_{i,t}^x(\widehat{\zeta}_{i,t})\rangle\notag\\
            &~+\sum_{t=1}^T\langle x_t-x,\sum_{i=1}^n\alpha_{i,t}G_{i,t}^x(\widehat{\zeta}_{i,t})-\sum_{i=1}^n\alpha_{i,t}g_{i,t}^x\rangle\Bigr\}\notag\\
            \leq&~\sup_{x\in\calX}\sum_{t=1}^T\langle x_t-x,sg_t^x-\sum_{i=1}^n\alpha_{i,t}G_{i,t}^x(\widehat{\zeta}_{i,t})\rangle\notag\\
            &~+\sum_{t=1}^TD_x\|\sum_{i=1}^n\alpha_{i,t}e_{i,t}^x\|\notag\\
            \leq&~\sup_{x\in\calX}\sum_{t=1}^T\langle -x,sg_t^x-\sum_{i=1}^n\alpha_{i,t}G_{i,t}^x(\widehat{\zeta}_{i,t})\rangle\notag\\
            &~+ \sum_{t=1}^T\langle x_t,sg_t^x-\sum_{i=1}^n\alpha_{i,t}G_{i,t}^x(\widehat{\zeta}_{i,t})\rangle+TD_xL_{x\xi}\epsilon\,. \label{eq:term2}
        \end{align}
        \end{small}%
        \textcolor{black}{Since $x_t$ is $\mathcal F_t$-measurable,~\eqref{eq:sg_expectation} and the tower property imply that the expectation of the term involving $x_t$ in~\eqref{eq:term2} is zero. Taking expectations in~\eqref{eq:term2} yields}
        \begin{small}
            \begin{align}
                &~\ev\Big[\sup_{x\in\calX}\sum_{t=1}^T\langle x_t-x,sg_t^x-g_t^x\rangle\Big]\notag\\
                \leq&~\ev\Big[\sup_{x\in\calX}\sum_{t=1}^T\langle -x,sg_t^x-\sum_{i=1}^n\alpha_{i,t}G_{i,t}^x(\widehat{\zeta}_{i,t})\rangle\Big]+TD_xL_{x\xi}\epsilon\notag\\
                \leq&~ D_x\ev\Big[\Big\|\sum_{t=1}^T\Big(sg_t^x-\sum_{i=1}^n\alpha_{i,t}G_{i,t}^x(\widehat{\zeta}_{i,t})\Big)\Big\|\Big]+TD_xL_{x\xi}\epsilon\notag\\
                \leq&~ D_x\sqrt{\ev\Big[\Big\|\sum_{t=1}^T\Big(sg_t^x-\sum_{i=1}^n\alpha_{i,t}G_{i,t}^x(\widehat{\zeta}_{i,t})\Big)\Big\|^2\Big]}+TD_xL_{x\xi}\epsilon\notag\\
               =&~ D_x\sqrt{\ev\Big[\sum_{t=1}^T\Big\|sg_t^x-\sum_{i=1}^n\alpha_{i,t}G_{i,t}^x(\widehat{\zeta}_{i,t})\Big\|^2\Big]}+TD_xL_{x\xi}\epsilon\notag\\
                \leq&~2D_xB_2\sqrt{T}+TD_xL_{x\xi}\epsilon\,,~\notag
            \end{align}
        \end{small}%
        where the third inequality follows from Jensen's inequality. \textcolor{black}{The last equality follows from the conditional mean-zero property in~\eqref{eq:sg_expectation} and the tower property: the error at an earlier round is measurable with respect to the history at a later round. Thus, for all $1\leq t^\prime\neq t\leq T$,}
        \begin{small}
            \begin{align*}
        \ev\Big[\Big\langle sg_t^x-\sum_{i=1}^n\alpha_{i,t}G_{i,t}^x(\widehat{\zeta}_{i,t}),
        sg_{t^\prime}^x-\sum_{i=1}^n\alpha_{i,{t^\prime}}G_{i,{t^\prime}}^x(\widehat{\zeta}_{i,{t^\prime}})\Big\rangle\Big]=0,
    \end{align*}
        \end{small}%
where $sg_{t^{\prime}}^x\in\partial_x F(x_{t^\prime},\alpha_{t^\prime})$. The last inequality holds since $\|sg_t^x\|\le B_2$ and $\|G_{i,t}^x(\widehat{\zeta}_{i,t})\|\le B_2$\,.

\end{proof}

\subsection{\textcolor{black}{Proof of Lemma~\ref{lem:dual_update_bound}}}\label{app:dual_update_bound}
\begin{proof}

By the update rule~\eqref{eq:global_alpha_update}, we have
   \begin{align}
       &~\| \alpha_{t+1}-\alpha\|^2\notag\\
       \leq&~ \|\alpha_t+\eta_\alpha g_t^\alpha-\alpha\|^2\notag\\
   =&~\|\alpha_t-\alpha\|^2+\eta_\alpha^2\|g_t^\alpha\|^2+2\eta_\alpha\langle \alpha_t-\alpha, g_t^\alpha\rangle\textcolor{black}{,}
   \end{align}
\textcolor{black}{where the inequality follows from the nonexpansiveness of the projection.}
By the definition of $C_i$, $|g_{i,t}^\alpha|\leq C_i$. It follows that
       \begin{small}
           \begin{align}
            &~\sum_{t=1}^T\langle \alpha-\alpha_t,g_t^\alpha\rangle\notag\\
            \leq&~ \frac{1}{2\eta_\alpha}\sum_{t=1}^T\left(\|\alpha_t-\alpha\|^2-\|\alpha_{t+1}-\alpha\|^2+\eta_\alpha^2\|g_t^\alpha\|^2\right)\notag\\
        \leq&~ \frac{1}{2\eta_\alpha}\|\alpha_1-\alpha\|^2+\frac{\eta_\alpha }{2}\sum_{t=1}^T\|g_t^\alpha\|^2\notag\\
        \leq&~ \frac{D_\alpha^{2}}{2\eta_\alpha}+\frac{\eta_\alpha T}{2}\sum_{i=1}^n C_i^2\,.\notag
       \end{align}
       \end{small}%

\end{proof}

\subsection{\textcolor{black}{Proof of Lemma~\ref{lem:dual_estimation_error}}}\label{app:dual_estimation_error}
\begin{proof}

Define $G_{i,t}^\alpha(\zeta):=f_i(x_t,\zeta)=\sup_{\xi\in\Xi}\phi_i(x_t,\zeta,\xi)$.
    The gradient of $F(x,\alpha)$ with respect to $\alpha_i$ can be computed as
        \begin{align}\label{eq:metric_gra_alpha}
            \nabla_{\alpha_i} F(x_t,\alpha_t)=\ev_{\zeta\sim\widehat{\bbP}_i}\left[ f_i(x_t,\zeta)\right]=\ev_{\zeta\sim\widehat{\bbP}_i}[G_{i,t}^\alpha(\zeta)].
        \end{align}
\textcolor{black}{Define the sampling error}
{\color{black}
\begin{align*}
M_{i,t}:=\nabla_{\alpha_i}F(x_t,\alpha_t)-G_{i,t}^\alpha(\widehat\zeta_{i,t}).
\end{align*}}
\textcolor{black}{Fresh sampling and~\eqref{eq:metric_gra_alpha} imply $\ev[M_{i,t}\mid\mathcal F_t]=0$. Moreover, $|G_{i,t}^\alpha(\widehat\zeta_{i,t})|\leq C_i$ and $|\nabla_{\alpha_i}F(x_t,\alpha_t)|\leq C_i$, so $|M_{i,t}|\leq 2C_i$.}
    Similarly, define the biased error by $e_{i,t}^\alpha:= g_{i,t}^\alpha-G_{i,t}^\alpha (\widehat{\zeta}_{i,t})$. It satisfies
    \begin{align}
            \|e_{i,t}^\alpha\|=&~\|(\ell(x_t,z_{i,t})-\rho_i c(\widehat{\zeta}_{i,t},z_{i,t}))\notag\\
            &~-(\ell(x_t,z_{i,t}^{\mathrm{opt}})-\rho_i c(\widehat{\zeta}_{i,t},z_{i,t}^{\mathrm{opt}}))\|
            \leq L_{\xi}^c\epsilon,
        \end{align}
    where $z_{i,t}^{\mathrm{opt}}\in\calZ_{i,t}^\star$ is the specific optimizer used before.
    Moreover, we obtain
        \begin{align}
        &\sup_{\alpha\in\Delta_{n-1}}\sum_{t=1}^T\langle \alpha-\alpha_t,\nabla_\alpha F(x_t,\alpha_t)-g_t^\alpha
\rangle\notag\\
=&\sup_{\alpha\in\Delta_{n-1}}\Bigl\{\sum_{t=1}^T\sum_{i=1}^n(\alpha_i-\alpha_{i,t})\bigl(\nabla_{\alpha_i}F(x_t,\alpha_t)-G_{i,t}^\alpha(\widehat{\zeta}_{i,t})\bigr)\notag\\
&+\sum_{t=1}^T\sum_{i=1}^n(\alpha_i-\alpha_{i,t})\bigl(G_{i,t}^\alpha(\widehat{\zeta}_{i,t})-g_{i,t}^\alpha\bigr)\Bigr\}\notag\\
&\leq \sup_{\alpha\in\Delta_{n-1}}\sum_{t=1}^T\sum_{i=1}^n(\alpha_i-\alpha_{i,t})\bigl(\nabla_{\alpha_i}F(x_t,\alpha_t)-G_{i,t}^\alpha(\widehat{\zeta}_{i,t})\bigr)\notag\\
&+nTL_{\xi}^c\epsilon.\label{eq:term4}
\end{align}
\textcolor{black}{Taking expectations over the algorithm's randomness in~\eqref{eq:term4} gives}
    \begin{small}
         \begin{align}
        &\ev\Big[\sup_{\alpha\in\Delta_{n-1}}\sum_{t=1}^T\langle \alpha-\alpha_t,\nabla_\alpha F(x_t,\alpha_t)-g_t^\alpha
\rangle\Big]\notag\\
\leq&~\ev\Big[\sup_{\alpha\in\Delta_{n-1}}\sum_{t=1}^T\sum_{i=1}^n(\alpha_i-\alpha_{i,t})\bigl(\nabla_{\alpha_i}F(x_t,\alpha_t)-G_{i,t}^\alpha(\widehat{\zeta}_{i,t})\bigr)\Big]\notag\\
&~+nTL_{\xi}^c\epsilon\notag\\
\overset{\textcolor{black}{(a)}}{=}&~\ev\Big[\sup_{\alpha\in\Delta_{n-1}}\sum_{t=1}^T\sum_{i=1}^n\alpha_i\bigl(\nabla_{\alpha_i}F(x_t,\alpha_t)-G_{i,t}^\alpha(\widehat{\zeta}_{i,t})\bigr)\Big]+nTL_{\xi}^c\epsilon\notag\\
\leq&~ \ev\Big[\sum_{i=1}^n\Big|\sum_{t=1}^T\bigl(\nabla_{\alpha_i}F(x_t,\alpha_t)-G_{i,t}^\alpha(\widehat{\zeta}_{i,t})\bigr)\Big|\Big]+nTL_{\xi}^c\epsilon\notag\\
\overset{\textcolor{black}{(b)}}{\leq}&~ 2\sqrt{T}\sum_{i=1}^nC_i+nTL_{\xi}^c\epsilon\,, \notag
    \end{align}
    \end{small}%
\textcolor{black}{where equality (a) uses the tower property of conditional expectation. Since $\alpha_{i,t}$ is $\mathcal F_t$-measurable,}
{\color{black}
\begin{align*}
\ev[\alpha_{i,t}M_{i,t}]
=\ev\big[\alpha_{i,t}\ev[M_{i,t}\mid\mathcal F_t]\big]=0.
\end{align*}
For each realization, the supremum is taken over $\alpha$, while the iterates $\alpha_t$ are held fixed. We can therefore take $\sum_{t,i}\alpha_{i,t}M_{i,t}$ outside the supremum and apply the preceding identity to obtain (a).
To prove (b), note that for $s<t$, $M_{i,s}$ is $\mathcal F_t$-measurable. Hence, the tower property also gives}
{\color{black}
\begin{align*}
\ev[M_{i,s}M_{i,t}]
=\ev\big[M_{i,s}\ev[M_{i,t}\mid\mathcal F_t]\big]=0.
\end{align*}
By the Cauchy--Schwarz inequality,}
{\color{black}
\begin{align*}
\ev\left[\left|\sum_{t=1}^T M_{i,t}\right|\right]
&\leq\sqrt{\ev\left[\left(\sum_{t=1}^T M_{i,t}\right)^2\right]}\\
&=\sqrt{\sum_{t=1}^T\ev[M_{i,t}^2]}
\leq 2C_i\sqrt T,
\end{align*}}%
\textcolor{black}{where the equality follows because the cross terms have zero expectation, as shown above, and the last inequality uses $|M_{i,t}|\leq 2C_i$. Summing over clients proves (b).}

\end{proof}
\endgroup

\fi

\section*{\textcolor{black}{Acknowledgment}}
\textcolor{black}{ChatGPT~\cite{openaiChatGPT} assisted with manuscript review, proofreading, and consistency checks. The authors retain full responsibility for all content and results.}

\begingroup
\let\originalbibitem\bibitem
\renewcommand{\bibitem}[2][]{%
  \par\color{black}%
  \ifthenelse{\equal{#2}{mohajerinesfahani2018datadriven}\OR
    \equal{#2}{kuhn2019wasserstein}\OR
    \equal{#2}{cherukuri2020cooperative}\OR
    \equal{#2}{nguyen2022generalization}\OR
    \equal{#2}{ibrahim2025fdrsvm}\OR
    \equal{#2}{boskos2024highconfidence}\OR
    \equal{#2}{Bertsekas2016Nonlinear}\OR
    \equal{#2}{sion1958general}\OR
    \equal{#2}{openaiChatGPT}}{\color{black}}{}%
  \ifthenelse{\equal{#2}{nguyen2022generalization}}{\color{black}}{}%
  \ifthenelse{\equal{#1}{}}{\originalbibitem{#2}}{\originalbibitem[#1]{#2}}%
}
\bibliographystyle{IEEEtran}
\bibliography{IEEEabrv,mybibfile}
\endgroup

\end{document}